%% file: main_arxiv.tex
\documentclass{article}
\usepackage{times}

\usepackage{microtype}
\usepackage{graphicx}
\usepackage{booktabs} 
\usepackage{subcaption}

\usepackage[hyphens]{url}
\usepackage{hyperref}

\usepackage[utf8]{inputenc} 
\usepackage[T1]{fontenc}    
\usepackage{hyperref}       
\usepackage{url}            
\usepackage{booktabs}       
\usepackage{amsfonts}       
\usepackage{nicefrac}       
\usepackage{microtype}      
\usepackage{algorithm,algorithmic,amsmath,amssymb,amsthm}
\usepackage{authblk}
\usepackage{graphicx}
\usepackage{wrapfig}
\usepackage{multicol}
\usepackage{multirow}
\usepackage{thm-restate}
\usepackage{amsmath, amssymb}
\usepackage{url}  
\usepackage{bm}

\DeclareMathOperator{\vecop}{vec}
\DeclareMathOperator*{\argmin}{arg\,min}

\usepackage{amssymb}
\usepackage{amsmath}
\usepackage{amsthm}
\usepackage{thmtools}
\usepackage{thm-restate}

\newtheorem{theorem}{Theorem}[section]
\newtheorem{lemma}[theorem]{Lemma}

\newcommand{\eat}[1]{}
\usepackage{color}

\usepackage[top=1in,bottom=1in,left=1in,right=1in]{geometry}

\title{Model Agnostic Multilevel Explanations}

\author{%
  Karthikeyan Natesan Ramamurthy, Bhanukiran Vinzamuri, Yunfeng Zhang, Amit Dhurandhar\\
  IBM Research, Yorktown Heights, NY USA 10598 \\
  \small{\texttt{knatesa@us.ibm.com, bhanu.vinzamuri@ibm.com, zhangyun@us.ibm.com, adhuran@us.ibm.com}}
}
\date{}

\begin{document}
\maketitle

\begin{abstract}

In recent years, post-hoc local instance-level and global dataset-level explainability of black-box models has received a lot of attention. Much less attention has been given to obtaining insights at intermediate or group levels, which is a need outlined in recent works that study the challenges in realizing the guidelines in the General Data Protection Regulation (GDPR).
In this paper, we propose a meta-method that, given a typical local explainability method, can build a multilevel explanation tree. The leaves of this tree correspond to the local explanations, the root corresponds to the global explanation, and intermediate levels correspond to explanations for groups of data points that it automatically clusters. The method can also leverage side information, where users can specify points for which they may want the explanations to be similar. We argue that such a multilevel structure can also be an effective form of communication, where one could obtain few explanations that characterize the entire dataset by considering an appropriate level in our explanation tree. Explanations for novel test points can be cost-efficiently obtained by associating them with the closest training points. When the local explainability technique is generalized additive (viz. LIME, GAMs), we develop a fast approximate algorithm for building the multilevel tree and study its convergence behavior. We validate the effectiveness of the proposed technique based on two human studies -- one with experts and the other with non-expert users -- on real world datasets, 
and show that we produce high fidelity sparse explanations on several other public datasets.
\end{abstract}

\begin{figure*}[h]
\centering
\includegraphics[width=\textwidth]{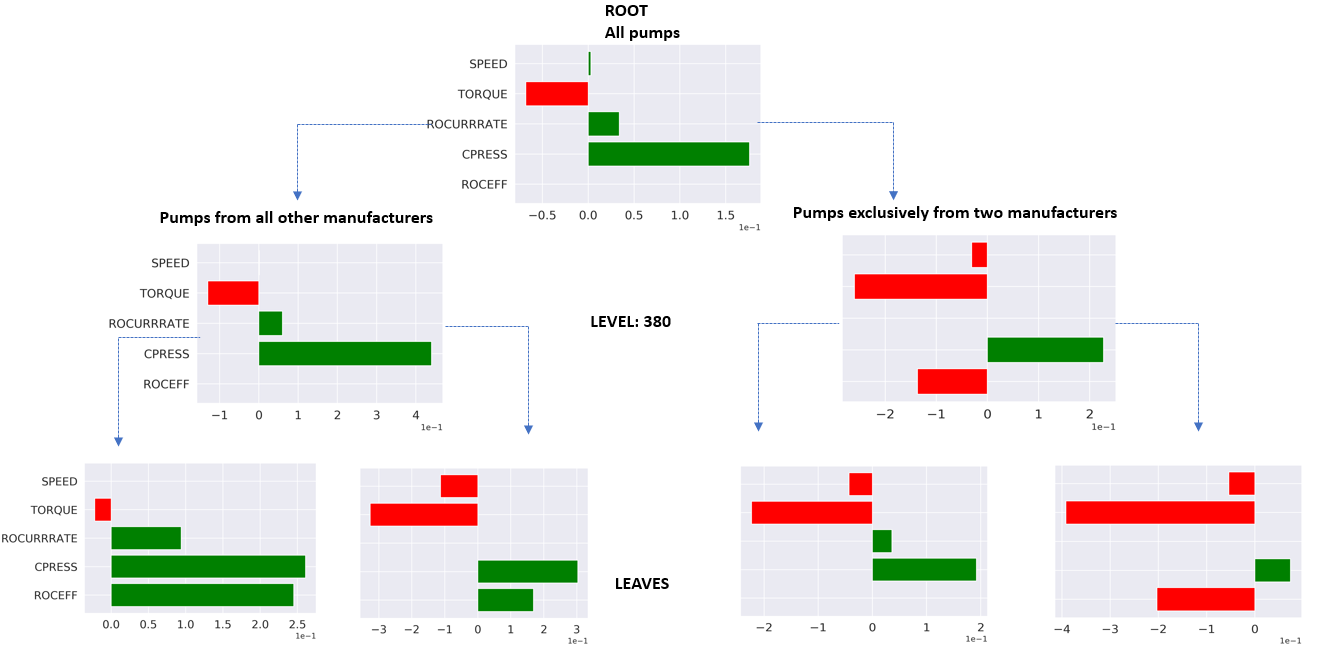}

\caption{Above we see an illustration of multilevel explanations generated by MAME for a real industrial pump failure dataset consisting of 2500 wells. We show three levels: the bottom level (four) leaves which correspond to example local explanations, the top level corresponds to one global explanation and an intermediate level corresponds to explanations for two groups highlighted by MAME. Based on expert feedback these intermediate explanations although explain the same type of pump (Progressive cavity), they have \emph{different manufacturers} resulting in some noticeable difference in behaviors. Another interesting aspect to note is that each level provides distinct enough information not subsumed by just local or global explanations, thus motivating the need for such multilevel explanations.}
\label{fig:mame_pump_intro}
\end{figure*}

\section{Introduction}

A very natural and effective form of communication is to first set the stage through high level general concepts and then only dive into more of the specifics \cite{goodcomm}. In addition, the transition from high level concepts to more and more specific explanations should ideally be as logical or smooth as possible \cite{goodcomm2,liptonPgoodexp}. For example, when you call a service provider there is usually an automated message trying categorize the problem at a high level followed by more specific questions until eventually if the issue isn't resolved you might be connected to a human representative who can further delve into details from that point on. Another example is when you are presenting a topic one usually starts at a high level providing some background and motivation followed by more specifics. A third example is when you visit the doctor with an ailment, you first have to fill forms which capture information at various (higher) levels of granularity such as your family's medical history followed by your personal medical history, after which a nurse may take your vitals and ask questions pertaining to your current situation. A doctor may perform further tests to pinpoint the problem.
In all these cases, information or explanations you provide at multiple levels enables others to obtain insight that otherwise may be opaque. Recent work \cite{gdpr-multilevel} has stressed the importance of having such multilevel explanations for successfully meeting the requirements of Europe's General Data Protection Regulation (GDPR) \cite{gdpr}. They argue that simply having local or global explanations may not be sufficient for providing satisfactory explanations in many cases. In fact, even in the widely participated FICO explainability challenge \cite{FICO} it was expected that one provides not just local explanations but also insights at the intermediate class level.

Given the omnipresence of multilevel explanations across various real world settings, we in this paper propose a novel model agnostic multilevel explanation (MAME) method that can take a local explainability technique such as LIME \cite{lime}
along with a dataset and can generate multiple explanations for each of the examples corresponding to different degrees of cohesion (i.e. parameter tying) between (explanations of) the examples, where each such degree determines a level in our multilevel explanation tree and is explicitly \emph{controllable}. At the extremes, the leaves would correspond to independent local explanations as would be the case using standard local explainability techniques (viz. LIME), while the root of the tree would correspond to practically a single explanation given the high degree of cohesion between all the explanations at this level. An illustration of this is seen in figure \ref{fig:mame_pump_intro}, where multilevel explanations generated by MAME for a real industrial pump failure dataset consisting of 2500 wells. We show three levels: the bottom level (four) leaves which correspond to example local explanations (amongst many), the top level corresponds to one global explanation and an intermediate level corresponds to explanations for two groups highlighted by MAME. The dotted lines indicate that the (explanation) nodes are descendants of the node above, but not direct children. Based on expert feedback these intermediate explanations although explain the same type of pump have different manufacturers resulting in some noticeable difference in behaviors. This is discussed in detail in Section \ref{OG}. Also note that each level provides distinct enough information not subsumed by just local or global explanations, thus motivating the need for such multilevel explanations.
Such explanations can thus be very insightful in identifying key characteristics that bind together different examples at various levels of granularity. Moreover, they can also provide \emph{exemplar based explanations} looking at the groupings at specific levels (viz. different pumps by the same manufacturer).

Our method can also take into account side information such as similarity in explanations based on class labels or user specified groupings based on domain knowledge for a subset of examples. Moreover, one can also use non-linear additive models going beyond LIME to generate local explanations. Our method thus provides a lot of \emph{flexibility} in building multilevel explanations that can be customized apropos a specific application.

We prove that our method actually \emph{forms a tree} in that examples merged in a particular level of the tree remain together at higher levels. The proposed \textit{fast approximate algorithm} for obtaining multilevel explanations is proved to converge to the exact solution.
We also validate the effectiveness of the proposed technique based on two
human studies – one with experts and the other with non-expert users – on real world datasets, and show that we produce high fidelity sparse explanations on several other public datasets.

\section{Related Work}
\label{sec:relatedwork}
We now look at some important lines of work in explainable AI. The most traditional direction is to directly build interpretable models such as rule lists or decision sets \cite{decl,twl} on the original data itself so that no post-hoc interpretability methods are required to uncover the logic behind the proposed actions. These methods however, may not readily give the best performance in many applications compared with a complex black-box model. 
There are works which try to provide local \cite{lime, montavon2017methods, unifiedPI,cem} as well as global explanations that are both feature and exemplar based \cite{maple}. Exemplar based explanations \cite{proto,l2c} essentially identify few examples that are representative of a larger dataset. The previous work however, uses distinct models to provide the local and global explanations where again consistency between the models could potentially be an issue.
Pedreschi \textit{et al.} \cite{pedreschi2018open} propose a method that clusters local rule-based explanations to learn a local to global tree structure. This however suffers from the fact that the local explanations which are typically generated independently may not be consistent for nearby examples making the process of understanding and eventually trusting the underlying model challenging. TreeExplainer~\cite{lundberg2020loal} is another local explanation method specifically for trees based on game theory concepts. The authors present five new methods that combine many local explanations to provide global insight, which allows to retain local faithfulness to the model. This work also has similar limitations as \cite{pedreschi2018open}. Tsang \textit{et al.} \cite{tsang2018can} propose a hierarchical method to study the change in the behavior of interactions for a local explanation from an instance level to across the whole dataset. One major difference is that they do not fuse explanations as we do as they go higher up the tree. Moreover, their notion of hierarchy is based on learning higher order interactions between the prediction and input features, which is different from ours.

Our  approach has relations to convex clustering \cite{chen2015convex,weylandt2019dynamic}, and its generalizations to multi-task learning \cite{yu2018simultaneous, yu2017multitask}. However, our goal is completely different (multilevel post-hoc explainability) and our methodology of computing and using local models that mimic black-box predictions is also different.

\begin{algorithm}[t]
    \caption{Model Agnostic Multilevel Explanation (MAME) method}
    \label{MAME}
\begin{algorithmic}
\STATE \textbf{Input:}  Dataset $x_1, ..., x_n$, black-box model $f(.)$, the coordinate wise map $g(.)$ and prior knowledge graph adjacency matrix $W$. 
\STATE i) Sample neighborhoods $\mathcal{N}_i$ for each example $x_i$.
\STATE ii) Construct matrix $D$ based on edge list $\mathcal{E}$.
\STATE iii) Initialize $\Theta^{(0)}$, $U^{(0)}$,$V^{(0)}$ and set $k=0$, multiplicative step-size $t$ (say to $1.01$), $\gamma^{(0)}=\epsilon$, grouping threshold $\tau$ (say to $1e-6$), $\rho$ (say to 2), and $tol.$ (say to $1e-6$).
\STATE  iv) Initialize a disjoint set with leaves of the multilevel the tree $S = \{1, \ldots, n\}$.
\WHILE{$\|V^{k+1}\| > tol.$}
\STATE Obtain $\Theta^{(k+1)}$ by solving (\ref{eqn:mame_admm_theta}).
\STATE Obtain $U^{(k+1)}$ by solving (\ref{eqn:mame_admm_U}).
\STATE Obtain $V^{(k+1)}$ by solving (\ref{eqn:mame_admm_V}) with $\beta$ set as $\gamma^{(k)}$.
\STATE Obtain $Z_1^{(k+1)}$,  $Z_2^{(k+1)}$ by solving (\ref{eqn:mame_admm_Z1}) and (\ref{eqn:mame_admm_Z2}).
\STATE For each edge $e_l = (i, j) \in \mathcal{E}$, if $\|v_l\| < \tau$, perform $\text{Union}(i, j)$.
\STATE $k=k+1$; $\gamma^{(k+1)}=\gamma^{(k)}*t$.
\ENDWHILE
\STATE iv) Recover the multilevel tree by keeping track of the disjoint set unions.
\STATE v) For every group of examples in each tree node, post-process to get representative explanations by optimizing  (\ref{eqn:mame}) with $\beta=0$ only for the  examples in the group. 
\end{algorithmic}
\end{algorithm}

\section{Method}
\label{sec:method}

Let $X\times Y$ denote the input-output space and $f:X\rightarrow Y$ a classification or regression function corresponding to a black-box classifier. For any positive integer $p$ let $g:\mathbb{R}^p\rightarrow \mathbb{R}^p$ denote a function that acts \emph{co-ordinate wise} on any feature vector $x\in X$. Thus, if $g(x)$ is an identity map then we recover $x$. However, $g(.)$ could be non-linear  and we could apply different non-linearities to different coordinates of $x$. If $\theta$ is a parameter vector of dimension $p$, then $l(x,\theta)=g(x)^T\theta$ can be thought of as a generalized additive model which could be learned based on the predictions of $f(.)$ for examples near $x$ thus providing a local explanation for $x$ given by $\theta$. Let $\psi(x,z)$ be the similarity between $x$ and $z$. This can be estimated for a distance function $d(.,.)$ as $\exp(-\gamma d(x, z))$. 
Let $(x_1, y_1), ..., (x_n, y_n)$ denote a dataset of size $n$, where the output $y_i$ may or may not be known for each example. Let $\mathcal{N}_i$ be the neighborhood of an example $x_i$, i.e. examples that are highly similar to $x_i$ formally defined as $\mathcal{N}_i=\{z\in X|\psi(x_i,z)\ge \gamma\}$ for a $\gamma\in [0,1]$ close to 1. In practice, $\mathcal{N}_i$
of size $m$ can be generated by randomly perturbing $x_i$ as done in previous works \cite{lime} $m$ times.
Given this we can define the following optimization problem:
\begin{align}
\label{eqn:mame}
\min_{\Theta^{(\beta)}} \sum_{i = 1}^n \sum_{z \in \mathcal{N}_i}&\psi\left(x_i,z\right) \left(f(z) - g(z)^T \theta_i\right)^2 + \alpha_i ||\theta_i||_1 + \beta \left(\sum_{i < j} w_{ij} \|\theta_i - \theta_j\|_2 \right)
\end{align}
where $\alpha_1, ...,~\alpha_n, ~ \beta \ge 0$ are regularization parameters, $w_{ij}\ge 0$ are custom weights and $\Theta^{(\beta)}$ is a set of $\theta_i$ $\forall i\in\{1, ..., n\}$ for a given $\beta$.

The first term in (\ref{eqn:mame}) tries to make the local models for each example to be as faithful as possible to the black-box model, in the neighborhood of the example. 
The second term tries to keep each explanation $\theta_i$ sparse. The third term tries to group together explanations. This in conjunction with the first term has the effect of trying to make \emph{explanations of similar examples to be similar}. Here we have the opportunity to \emph{inject domain knowledge} by creating a prior knowledge graph with adjacency matrix $W$. The edge weights $w_{ij}$ can be set to high values for pairs of examples that we consider to have similar explanations, while setting zero weights for other pairs if we believe their explanations will be different. 

We solve the above objective for different values of $\beta$, wherein $\beta=0$ corresponds to the leaves of the tree, with each leaf representing a  training example and its local explanation. At $\beta = 0$, (\ref{eqn:mame}) decouples to $n$ optimizations, corresponding to the LIME explanations. $\beta$ can be adaptively increased from $0$ resulting in progressive grouping of explanations (and hence the corresponding examples) forming higher levels of the tree. The grouping happens because $\theta_i$ and $\theta_j$ with a non-zero $w_ij$ are encouraged to get closer as $\beta$ increases in (\ref{eqn:mame}). The intermediate levels hence correspond to disjoint clusters of examples with their representative explanations. The root of the tree obtained at a high $\beta$ value represents the global explanation for the entire dataset. Our formulation differs from Two Step (one of our baselines in Section~\ref{sec:experiments}) which does convex clustering on LIME-based local explanations minimizing $\|\Omega-\tilde{\Theta}\|^{2}_F + \beta \left(\sum_{i < j} w_{ij} \|\tilde{\theta}_i - \tilde{\theta}_j\|_2 \right)$ where $\Omega \in \mathbb{R}^{p \times n}$ are LIME-based local explanations for $n$ instances. Two Step also results in a multilevel tree, although it does not explicitly ensure fidelity to the black-box model predictions.

\subsection{Optimization Details}
\label{optd}
We solve the optimization  in (\ref{eqn:mame}) using ADMM~\cite{boyd2011distributed} by posing (\ref{eqn:mame}) as
\begin{gather}
\min_{\Theta^{(\beta)}, U^{(\beta)}, V^{(\beta)}} \sum_{i = 1}^n \sum_{z \in \mathcal{N}_i}\psi\left(x_i,z\right) \left(f(z) - g(z)^T \theta_i\right)^2 + \alpha_i ||u_i||_1 + \beta \left(\sum_{e_l \in \mathcal{E}} w_l \|v_l\|_2 \right),\nonumber \\
\text{ such that } \Theta = U, \text{ } \Theta D = V.
\label{eqn:mame_1}
\end{gather} The augmented Lagrangian with scaled dual variables is
\begin{gather}
\min_{\Theta, U, V, Z_1, Z_2} \sum_{i = 1}^n \sum_{z \in \mathcal{N}_i}\psi\left(x_i,z\right) \left(f(z) - g(z)^T \theta_i\right)^2 + \alpha_i ||u_i||_1 + \beta \left(\sum_{e_l \in \mathcal{E}} w_l \|v_l\|_2 \right) + \nonumber \\
\label{eqn:mame_alm}
\frac{\rho}{2}\|\Theta - U + Z_1 \|_F^2 + \frac{\rho}{2}\|\Theta D - V + Z_2 \|_F^2. 
\end{gather} 

Here, $U \in \mathbb{R}^{p \times n}$, and $V \in \mathbb{R}^{p \times |\mathcal{E}|}$ are the auxiliary variables, and $\mathcal{E}$ is the list of edges in the prior knowledge graph with non-zero weights. The columns of $U$ and $V$ are denoted by $u_i$ and $v_l$ respectively. $u_i$ corresponds to the same column in $\Theta$ ($\theta_i$), and $D \in \mathbb{R}^{n \times |\mathcal{E}|}$ acts on $\Theta$ to encode differences in their columns. For example, the column of $D$ that encodes $\theta_i - \theta_j$ will contain $1$ at row $i$ and $-1$ at row $j$. $Z_1$ and $Z_2$ are the scaled dual variables. This reformulation is inspired by \cite{weylandt2019dynamic}.

The ADMM iterations for a given value of $\beta$ are:
\begin{gather}
\Theta^{(k+1)}= \argmin_{\Theta} \sum_{i = 1}^n \sum_{z \in \mathcal{N}_i}\psi\left(x_i,z\right) \left(f(z) - g(z)^T \theta_i^{(k)}\right)^2 + \frac{\rho}{2}\|\Theta^{(k)} - U^{(k)} + Z_1^{(k)} \|_F^2 + \nonumber\\
\label{eqn:mame_admm_theta}
 \frac{\rho}{2}\|\Theta^{(k)} D - V^{(k)} + Z_2^{(k)} \|_F^2,\\
\label{eqn:mame_admm_U}
U^{(k+1)} = \argmin_{U} \sum_i \alpha_i ||u_i^{(k)}||_1 +
\frac{\rho}{2}\|\Theta^{(k+1)} - U^{(k)} + Z_1^{(k)} \|_F^2, \\
\label{eqn:mame_admm_V}
V^{(k+1)} = \argmin_{V} \beta \left(\sum_{e_l \in \mathcal{E}} w_l \|v_l^{(k)}\|_2 \right) + \frac{\rho}{2}\|\Theta^{(k+1)} D - V^{(k)} + Z_2^{(k)} \|_F^2, \\
\label{eqn:mame_admm_Z1}
Z_1^{(k+1)} = Z_1^{(k)} + \Theta^{(k+1)} - U^{(k+1)}, \\
\label{eqn:mame_admm_Z2}
Z_2^{(k+1)} = Z_2^{(k)} + \Theta^{(k+1)} D - V^{(k+1)}.
\end{gather} 

Since (\ref{eqn:mame_admm_theta})-(\ref{eqn:mame_admm_Z2}) should be solved for progressively increasing  values of $\beta$, we adopt the idea of Algorithmic Regularization (AR) \cite{weylandt2019dynamic} to run (\ref{eqn:mame_admm_theta})-(\ref{eqn:mame_admm_Z2}) only once for each value of $\beta$, and warm-start the next set of ADMM iterations with the estimate for the previous $\beta$ value. The $\beta$ values are obtained by initializing to a small $\epsilon$ and multiplying it by a step size $t (> 1.0)$ for the next $k$. We will denote these approximate solutions as $\Theta^{(k)}$, where $k$ corresponds to the index of the set of $\beta$ values. The detailed algorithm for obtaining the multilevel tree using MAME is described in Algorithm \ref{MAME}.

The exact solution where (\ref{eqn:mame_admm_theta})-(\ref{eqn:mame_admm_Z2}) are run until convergence for each $\beta$ value, will be denoted as $\Theta^{(\beta)}$. We show that the approximate solution converges to the exact one in the following sense, which is proved in the appendix. Note that the Theorem holds true for the approximate solutions $\{\Theta^{(k)}\}$ without the post-processing step (v) in Algorithm \ref{MAME}.

\begin{restatable}{theorem}{mamehausdorffmain} \label{thm:mamehausdorffmain}
As $(t, \epsilon) \to (1, 0)$, where $t$ is the multiplicative step-size update and $\epsilon$ is the initial regularization level, the sequence of AR-based primal solutions $\{\Theta^{(k)}\}$, and the sequence of exact primal solutions $\{\Theta^{(\beta)}\}$ converge in the following sense.
\begin{align}
    \max\Big\{ 
    E_\beta \left(\inf_k \|\Theta^{(k)} - \Theta^{(\beta)} \|\right),
    E_k \left(\inf_\beta \|\Theta^{(k)} - \Theta^{(\beta)} \| \right)
    \Big\}
    \xrightarrow{(t, \epsilon) \to (1, 0)} 0
\end{align}
\end{restatable}

The approximation quality and timing comparisons between the AR-based and exact methods are in the appendix.

We now show that our method actually forms a tree in that explanations of examples that are close together at lower levels will remain at least equally close at higher levels.
\begin{lemma}[Non-expansive map for exact solutions]
If $\beta_1, ..., \beta_k$ are regularization parameters for the last term in (\ref{eqn:mame}) for $r$ consecutive levels in our multilevel explanations where $\beta_1=0$ is the lowest level with $\theta_{i, s}$ and $\theta_{j, s}$ denoting the (globally) optimal coefficient vectors (or explanations) for $x_i$ and $x_j$ respectively corresponding to level $s\in\{1, ..., r\}$, then for $s>1$ and $w_{ij}>0$ we have 
\begin{equation*}
\|\theta_{i, s} - \theta_{j, s}\|_2 \le \|\theta_{i, s-1} - \theta_{j, s-1}\|_2.
\end{equation*}
\end{lemma}
The proof of this lemma is available in the appendix. 

\section{Experiments}
\label{sec:experiments}

We now evaluate our method based on three different scenarios. The first is a case study involving human experts in the Oil \& Gas industry. Insights provided by MAME were semantically meaningful to the experts both when they did and did not provide side information. Second, we conducted a user study with data scientists based on a public loan approval dataset who were not experts in finance. We found that our method was significantly better at providing insight to these non-experts compared with Two Step -- which is hierarchical convex clustering \cite{chen2015convex,weylandt2019dynamic} where a median explanation is computed for each cluster -- and Submodular Pick LIME (SP-LIME) \cite{lime}. Data scientists are the right catcher for our method as a recent study \cite{umang} claims that explanations from AI based systems in most organizations are first ingested by data scientists. Third, we show quantitative benefit of our method in terms of two measures defined in Section \ref{QE}. The first measure, \textit{Generalized Fidelity}, quantifies how well we can predict the class for a test point using the explanation or feature importances of the closest training point and averaging this value over all test instances at different levels of aggregation. We find here too that our method is especially good when we have to explain the dataset based on few explanations, which is anyways the most desirable for human consumption. The second measure, \textit{Feature importance rank correlation}, shows the correlation between the ranks of feature importances between the black-box model and the explanation method. We observe that our proposed MAME method was superior to Two Step and SP-LIME in identifying the important features.

\subsection{Oil \& Gas Industry Dataset}
\label{OG}
We perform a case study with a real-world industrial pump failure dataset (classification dataset) from the Oil \& Gas industry. The pump failure dataset consists of sensor readings acquired from 2500 oil wells over a period of four years that contains pumps to push out oil. These sensor readings consist of measurements such as speed, torque, casing pressure (CPRESS), production (ROCURRRATE) and efficiency (ROCEFF) for the well along with the type of failure category diagnosed by the engineer. In this dataset, there are three major failure modes: Worn, Reservoir and Pump Fit. Worn implies the pump is worn out because of aging. Reservoir implies the pump has a physical defect. Pump Fit implies there is sand in the pump. From a semantics perspective, there can be seven different types of pumps in a well which can be manufactured separately by fourteen different vendors. We are primarily interested in modeling reservoir failures as they are the more difficult class of failures to identify. The black-box classifier used was a 7-layer multilayer perceptron (MLP) with parameter settings recommended by scikit-learn MLPClassifier. The dataset had 5000 instances 75\% of which were used to train the model and the remaining 25\% was test.

We conducted two types of studies here. One where we obtained explanations without any side information from the experts and the other where we were told certain groups of pumps that should exhibit similar behavior and hence should have similar explanations for their outputs.

\noindent\textbf{Study 1- Expert Evaluation:} In this study, we built the MAME tree on the training instances. We picked a level which had 4 clusters guided by expert input given the dataset had 4 prominent pump manufacturer types. In Figure \ref{fig:mame_pump_intro}, in the introduction we show the root and example leaves and the level with 4 clusters (level 380), where two of these clusters are shown that the expert felt were semantically meaningful. The expert said that the two clusters had semantic meaning in that, although both clusters predominantly contained the same type of Progressive Cavity pump (main pump type of interest), they were produced by different manufacturers and hence, had somewhat different behaviors. Two of the manufacturers were known to have better producing pumps in general which corresponds to the explanation on the right compared to that on the left which had pumps from all other mediocre producing manufacturers, which was consistent with on-field observations. Hence, the result uncovered by MAME without any additional semantic information gave the expert more trust in the model that was built.

\noindent\textbf{Study 2- Expert Evaluation with Side Information:}
In this study, the expert provided us a grouping of pumps
that should have similar explanations. Given the flexibility of MAME we were able to incorporate this knowledge ($W$) (see Eqn.~(\ref{eqn:mame})) into our optimization.
\begin{figure}[!ht]
\centering
\includegraphics[width=0.7\textwidth]{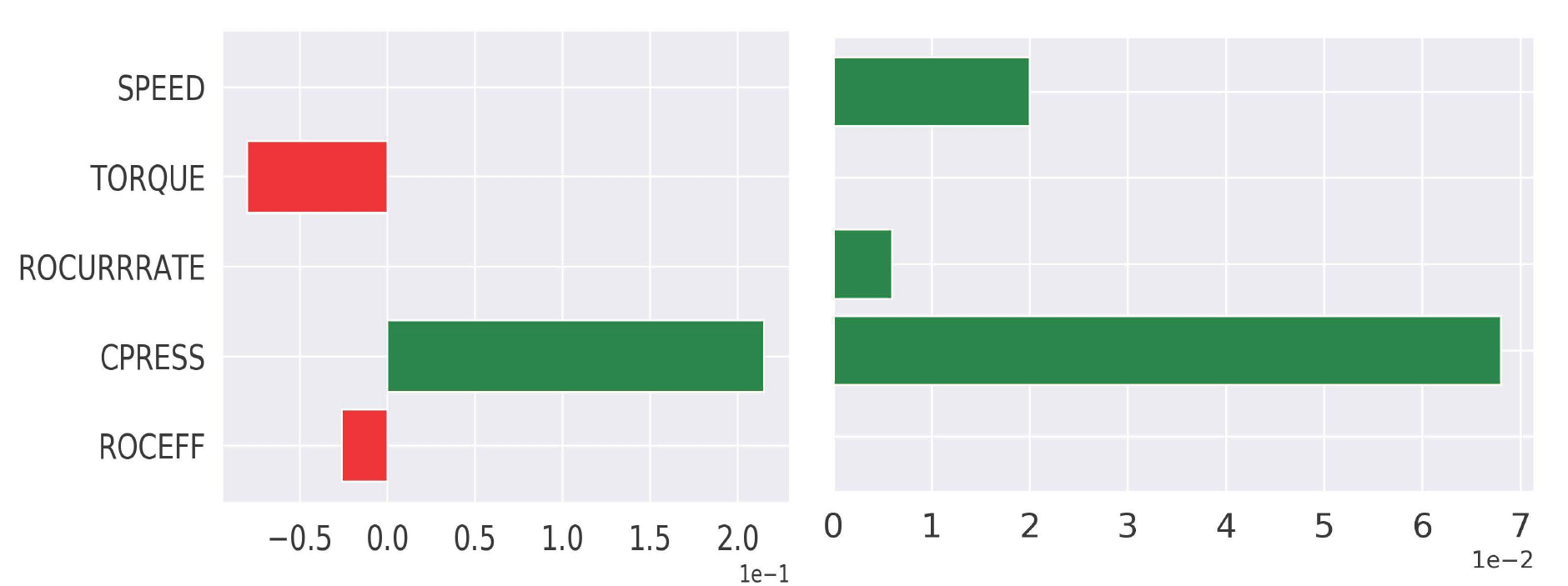}
\caption{Representative explanations for 2 clusters learned from MAME after incorporating expert provided knowledge which could segregate all the pumps into homogeneous pump type clusters, where each cluster could further be interpreted as consisting of primarily \emph{producing (left)} or \emph{well down (non-producing) (right)} pumps.}
\label{fig:pumpsemanticsmame}
\end{figure}

As explained above, based on expert input, we picked the level in the tree which had four clusters. Two of these clusters were small and had a mixed set of pumps. However, two of the bigger homogeneous clusters shown in Figure \ref{fig:pumpsemanticsmame} were of interest to the expert. The expert provided the following insights: \emph{non-producing} (Well-down category) pumps are more likely to be used in a run-to-failure scheme where they are used rigorously (i.e. at higher speed, production and casing pressure) than they would be run regularly which explains why these factors impact reservoir type failures here.

Producing pumps on the other hand keep producing oil adequately for longer periods while operating at optimal efficiency (slightly lower than max efficiency). Also operating these producing pumps with lower torque (caused by the helical rotor in the pump) can elongate their operational lifetime and reduce likelihood for imminent failure. 

These insights are consistent with our explanations and validates the fact that our method is able to effectively model the provided side information.

\subsection{User Study}
\label{sec:user_study}
To further evaluate MAME's ability to accurately capture high-level concepts for a given domain, we conducted another user study with 30 data scientist from three different technology companies. In contrast to the oil and gas case study, these data scientists are \textit{not} domain experts in the target task of the study (approving loans). We hypothesized that by showing the explanations for the high-level clusters found in a dataset, non-domain-experts can very quickly learn the critical relationships between predictors and outcome variables and start to make accurate predictions. Furthermore, we hypothesized that MAME produces better high-level explanations than the aforementioned Two Step method and the SP-LIME method \cite{lime}, which incrementally picks representative as well as diverse explanations based on a submodular objective. To evaluate this hypothesis, we created three conditions each with a different explanation method, and randomly assigned 10 participants to each condition. Example screenshots of the web based trials are provided in the appendix.

In the study, we asked the participants to play the role of a loan officer who needs to decide whether to approve or reject a person's loan request based on that person's financial features. We used the HELOC dataset \cite{FICO} as the basis for this task. The dataset contained 23 predictors, all of them were continuous variables about a person's financial record such as the number of loan-payment delinquencies, percentage of revolving balance, etc. One predictor, external credit score, was removed since it was essentially a summary of other variables and accounted for much of the variability in the outcome variable. The outcome variable was binary, with 0 indicating default on a loan and 1 indicating full repayment on a loan. 75\% of the dataset was used for training a random forest model with 100 trees, while the remaining was used to randomly pick instances for our user study. Based on the training set MAME, SP-LIME and Two Step were made to produce four representative explanations that would characterize the whole dataset and consequently partition it into four clusters. We chose the number four because we did not want to overwhelm the participants with too many explanations/clusters, and because we saw that the data could be partitioned into four types of applicants with each type having enough representation.

A single trial of the task proceeded as follows. First, the high-level explanation and the 22 financial features of a loan applicant were shown to the participant. Figure~\ref{fig:heloc_mame} shows the explanation generated using MAME. The top left graph shows that the classifier groups loan applicants into four clusters (y-axis) based on several key features (x-axis). The color of the squares indicates the average contribution of a feature to the classifier's prediction for a cluster, with orange colors indicating increasing probability of full repayment, whereas blue indicating decreasing probability. The numbers in the colored squares show the average feature value for that cluster. The top right graph shows the average of the classifier's predicted probability of repayment for loan applicants of each cluster. The bottom graph shows the overall importance of the features used in the top left graph. The longer the bar, the more important it is to the classifier. The same set of visual explanations was used throughout each condition as it does not change from trial to trial. To maintain the same level of explanation complexity across conditions, in the SP-LIME and Two Step conditions, we also used the algorithms to find four high-level clusters from the training data. However, the important features and the repayment probability of the clusters are substantially different from the MAME explanation.
\begin{figure}
    \centering
    \includegraphics[width=0.7\columnwidth]{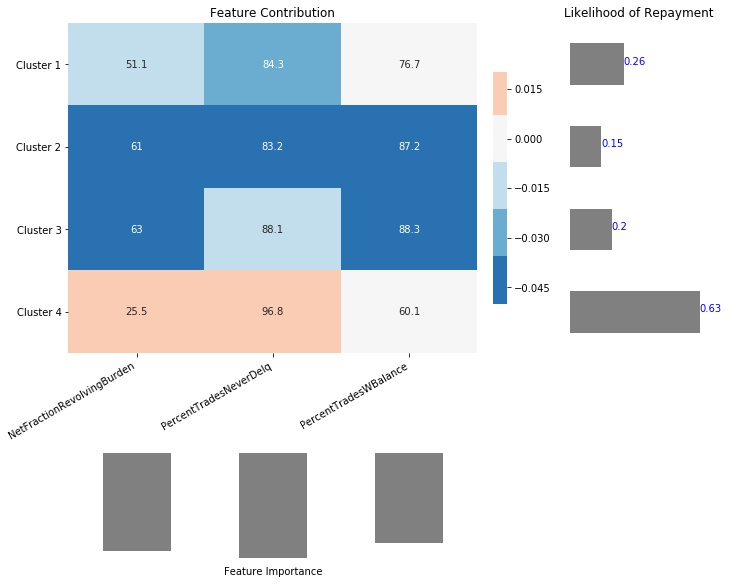}
    \caption{The high-level visual explanation shown in the user study for the MAME condition.}
    \label{fig:heloc_mame}
\end{figure}

After showing the necessary information, the participant was asked two questions for each trial. First, they were asked which of the four clusters displayed in the explanation graph is most similar to the loan applicant; second, they were asked to indicate their guess of the classifier's predicted repayment probability for the applicant on a scale of 0 to 100\% at the interval of 5\%.

Each participant completed 10 trials, which are randomly sampled from the test set. Their results are summarized in Figure~\ref{fig:heloc_results}. The left graph shows the mean squared error (MSE) between the classifier's probability prediction and the participant's guess of that prediction. MAME outperformed the other two methods since the participants had much lower MSE in this condition, averaging around just $0.036$. The right graph shows the number of trials (out of 10) in which the participant chose the correct most-similar cluster. The ground truth for this question was determined using nearest neighbor on the feature contribution of the given applicant and that of each cluster. As can be seen, MAME again outperformed other methods substantially. Participants in this condition were able to correctly find the most similar cluster for 7 out of 10 trials, while those in other conditions found on average less than 3 correct clusters. Both sets of results suggest that MAME produced more accurate high-level clusters than Two Step and SP-LIME methods. In addition, the very low squared difference suggests that non-experts can use MAME's high-level explanations to quickly understand how the model works and make reasonable predictions.
\begin{figure}
    \begin{minipage}[t]{0.50\textwidth}
    \centering
    \includegraphics[width=0.70\textwidth]{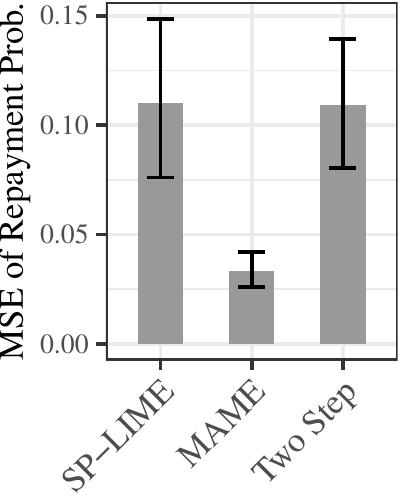}
    \end{minipage}
    \begin{minipage}[t]{0.50\textwidth}
    \centering
    \includegraphics[width=0.70\textwidth]{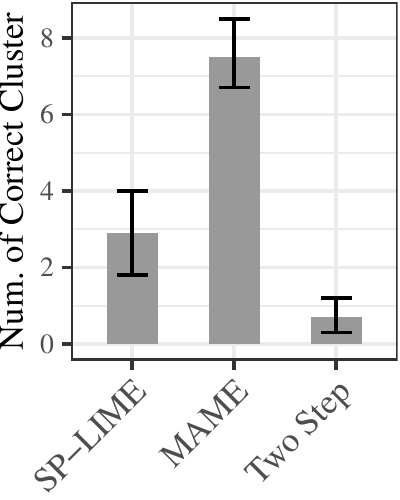}
    \end{minipage}
    \caption{Results of the user study. The left graph shows the mean squared error between the classifier's predicted repayment probability and the participants' guess of the classifier's prediction (lower is better). The right graph shows the average number of trials (out of 10) in which the participant correctly chose the most similar cluster (higher is better). The error bars show 95\% confidence intervals.}
    \label{fig:heloc_results}
\end{figure}

\begin{figure*}[t]
     \centering
     \footnotesize
     \begin{subfigure}[b]{0.49\textwidth}
        \centering
         \includegraphics[width=0.8\textwidth]{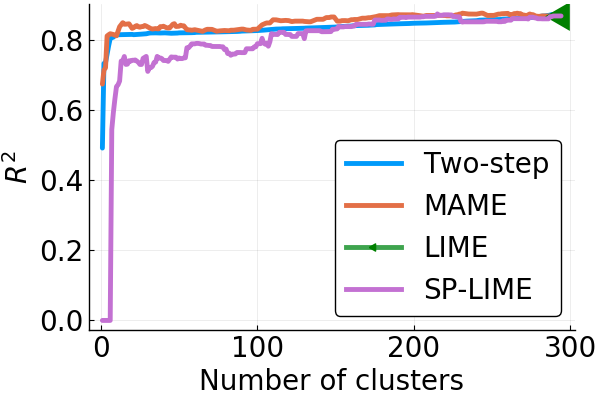}
         \caption{Auto MPG}
         \label{fig:autompgs_MLPR_fidelity}
     \end{subfigure}
     \begin{subfigure}[b]{0.49\textwidth}
        \centering
         \includegraphics[width=0.8\textwidth]{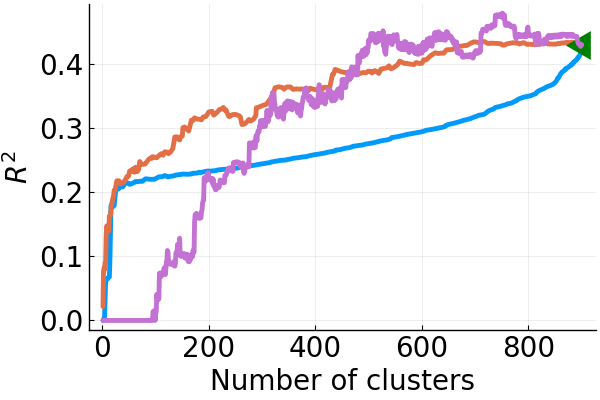}
         \caption{Retention}
         \label{fig:retention_1200_MLPC_fidelity}
     \end{subfigure}
     \begin{subfigure}[b]{0.49\textwidth}
         \centering
         \includegraphics[width=0.8\textwidth]{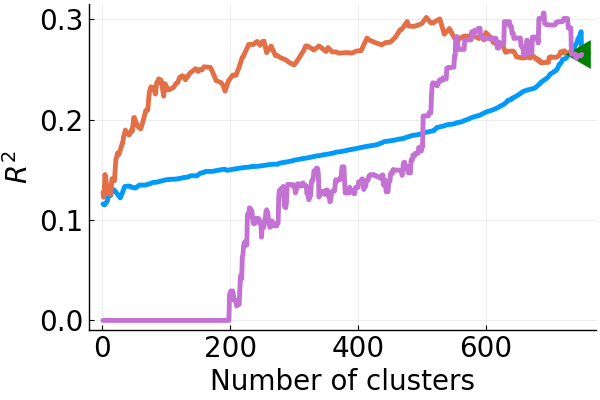}
         \caption{HELOC}
         \label{fig:heloc_MLPC_fidelity}
     \end{subfigure}
     \begin{subfigure}[b]{0.49\textwidth}
         \centering
         \includegraphics[width=0.8\textwidth]{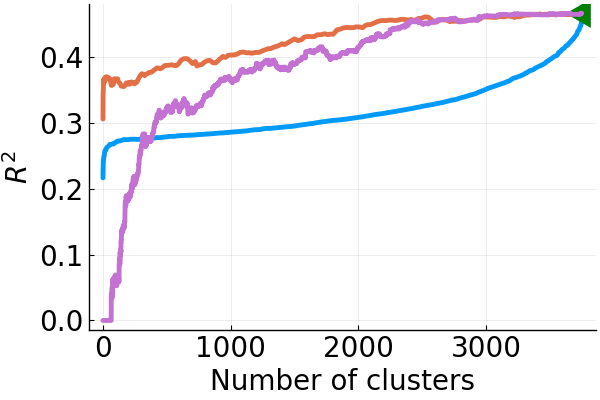}
         \caption{Waveform}
         \label{fig:waveform_MLPC_fidelity}
     \end{subfigure}
     \begin{subfigure}[b]{\textwidth}
         \centering
         \includegraphics[width=0.4\textwidth]{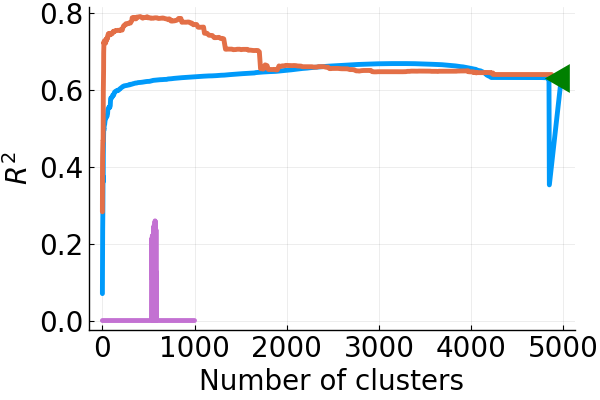}
         \caption{ATIS}
         \label{fig:atis_MLPC_fidelity}
     \end{subfigure}     
        \caption{Generalized fidelity measure for the datasets with MLP classifier. The legend in (a) is applicable to  all figures.}
        \label{fig:generalized_fidelity_MLP}
\end{figure*}

\begin{table}[h]
\begin{center}

\begin{tabular}{|c | c | c | c |}
\hline
\textit{Dataset} & \multicolumn{1}{c|}{LIME} & \multicolumn{1}{c|}{Two Step} & \multicolumn{1}{c|}{MAME} \\
\hline
\hline
\textit{Auto MPG}  & 0.4800                      & 0.4667                       & \textbf{0.5267}                   \\
\textit{Retention} & 0.5158                      & 0.5368                       & \textbf{0.5474}                   \\
\textit{HELOC}     & 0.6623                      & 0.6537                      & \textbf{0.6883}                 \\
\textit{Waveform}  & 0.1905                      & \textbf{0.5333}                       & 0.2952 \\
\textit{ATIS}  & 0.6976                      & 0.7089                       & \textbf{0.7547} \\
\hline
\end{tabular}

\caption{Kendall's rank correlation coefficient between ranks of black-box model feature importance (Random Forest Regressor/Classifier) and ranks of feature importance of the explanation methods.}
\label{tab:feature_rank_correlation}
\end{center}
\end{table}

\subsection{Quantitative Evaluation}
\label{QE}
\noindent\textbf{Quantitative Metrics:} We wish to primarily answer two questions quantitatively: (i) How trustworthy the learned local explanation models are at different levels in the tree? (ii) How faithful are the explanation models to the black-box that they explain?

\noindent a) \emph{Generalized Fidelity:} Here, for every example in the test set we find the closest example based on euclidean distance ($1-$nearest neighbor) in the set that the explanation models were built and use that explanation as a proxy for the test example. We then compute the $R^2$ of the prediction from the linear explanation model with respect to the black-box model's prediction. If this is high, it implies that the local models at that particular level capture the behavior of the black-box model accurately and thus can be trusted.  For our method and Two Step we do this for each level in the tree and compare the levels that have the same number of groups/clusters as the depth of the trees may vary. We also compute this measure for SP-LIME by varying the number of representative explanations from $1$ to the total number of training examples. The closest training example for each test point in this case is chosen from the list of representatives picked by SP-LIME, and the representative linear explanation models are used to compute the $R^2$. We remark that the $1-$nearest neighbor association is meaningful since LIME explanations (which also form the basis for Two Step and MAME) are optimized to work well in the neighborhood of each training data point.

\noindent b) \emph{Feature importance rank correlation:} We compute the feature importances of coefficients with LIME, Two Step and MAME methods. For LIME, the importance score of a feature $j$ is defined as $\sqrt{\sum_{i=1}^n |\theta_{ij}|}$, where $n$ is the number of samples \cite{lime}. For Two Step and MAME, we define the importance score similarly by including all the levels as $\sqrt{\sum_{i=1}^n \sum_{k=1}^l |\theta_{ij}^{(k)}|}$ where $l$ is the number of levels. We then rank the features in descending order of importance scores, and compute rank correlations after similarly ranking the black-box model feature importances. Note that this comparison can be performed only for black-box models that can output feature importance scores.

\subsection{Evaluation on Public Datasets}
\label{sec:eval_public_data}
We demonstrate the proposed methods using the quantitative metrics with several publicly available datasets: \textit{Auto MPG} \cite{uci}, \textit{Retention}\footnote{We generated the data using code in \url{https://github.com/IBM/AIX360/blob/master/aix360/data/ted_data/GenerateData.py}} \cite{aix360-sept-2019}, \textit{Home Line Equity Line of Credit (HELOC)} \cite{FICO}, \textit{Waveform} \cite{uci}, and \textit{Airline Travel Information System (ATIS)} datasets.

The Auto MPG dataset has a continuous outcome variable whereas the rest are discrete outcomes. The (number of examples, number of features) in the datasets are: Auto MPG (392, 7), Retention (1200, 8), HELOC (1000, 22), Waveform (5000, 21), and ATIS (5871, 128). We used randomly chosen 75\% of the data for training the black-box model and the explanation methods, and the rest 25\% for testing. The black-box models trained are Random Forest (RF) Regressor/Classifier, and Multi-Layer Perceptron (MLP) Regressor/Classifier. Only the Auto MPG dataset used a regression black-box, while rest used classification black-box models. The RF models used between 100 and 500 trees, whereas the MLP models have either 3 or 4 hidden layers with 20 to 200 units per layer. The labels for the explanation models are the predictions of the black-box models. For regression black-boxes, these are directly the predictions, whereas for classification black-boxes, these are the predicted probabilities for a specified class.

When running LIME and MAME methods, the neighborhood size $|\mathcal{N}_i|$ in (\ref{eqn:mame}) is set to $10$, neighborhood weights $\psi(x_i, z)$ are set using a Gaussian kernel on $\|x_i-z\|^2_2$ with a automatically tuned width, and the $\alpha_i$ values  in (\ref{eqn:mame}) are set to provide  explanations with $5$ non-zero values when $\beta=0$. 

For Two Step and MAME, the labels $f(x_i) $ are sorted and the $w_{ij}$ (see Section \ref{sec:method}) are set to $1$ whenever $f(x_i)$ and $f(x_j)$ are right next to each other in this sorted list, else $w_{ij} = 0$. This enforces the simple prior knowledge that explanations for similar black-box model predictions must be similar. This prior provided good results with our datasets. For both Two Step and MAME, we used post-processed explanations (see Section \ref{sec:method}). More details on the datasets, the black-box model parameters and performances, the classes chosen for explanation, and the hyper-parameters for the explanation models are available in the appendix.

The generalized fidelity measures for  the datasets for MLP black-boxes is given in Figures \ref{fig:autompgs_MLPR_fidelity} - \ref{fig:atis_MLPC_fidelity}. The MAME method outperforms Two Step and SP-LIME for small number of clusters (which is more important from an explainability standpoint) for the  classification datasets. For \textit{Auto MPG}, Two Step and MAME perform similarly and outperform SP-LIME. However, this dataset is the simplest of the lot. For \textit{ATIS}, SP-LIME took an inordinately long time, so we ran it only for 1 to 1000 representative explanations. Generalized fidelity for RF black-boxes are provided in the appendix.

The feature importance ranks with respect to RF black-boxes are provided in Table \ref{tab:feature_rank_correlation} for the  various datasets. Except for \textit{Waveform}, MAME's feature importance ranks are closer to that of the black-box model. Since feature importances are not output by MLP models, we do not consider those for this experiment.

\section{Conclusion}

In this paper, we have provided a meta explanability approach that can take a local explanability method such as LIME and produce a multi-level explanation tree by jointly learning the explanations with different degrees of cohesion, while at the same time being fidel to the black-box model. We have argued based on recent works as well as through expert and non-expert user studies that such explanations can bring additional insight not conveyed readily by just global or local explanations. We have also shown that when one desires few explanations to understand the model and the data as typically would be the case, our method creates much more fidel explanations compared with other methods. We have also made our algorithm scalable by proposing principled approximations to optimize the objective.

Our current method can build multi-level explanations based on linear or non-linear local explanability methods that have a clear parametrization. In the future, it would be interesting to extend the method to other non-parametric local explanation methods such as contrastive/counterfactual methods. A key there would be to somehow ensure that the perturbations for individual instances are to some degree common across instances that would be grouped together. From an HCI perspective, it would be interesting to see if such a tree facilitates back and forth communication where a user may adaptively want more and more granular explanations.

\small
\bibliography{mainrefs}
\bibliographystyle{IEEEtran}

\input{appendix_arxiv.tex}

\end{document}

%% file: appendix_arxiv.tex
\normalsize
\appendix

\setcounter{figure}{5} 
\setcounter{equation}{8}

\section{Additional Details on Public Datasets}

\noindent \paragraph{Auto MPG:} This dataset is obtained from \url{https://archive.ics.uci.edu/ml/datasets/Auto+MPG}. The features correspond to the various attributes of a model of a car, and the outcome is miles per gallon. This is the only regression dataset used.

\noindent \paragraph{Retention:} This is a synthetic dataset that is generated using \url{https://github.com/IBM/AIX360/blob/master/aix360/data/ted_data/GenerateData.py}. The features correspond to job position, organization, performance, compensation, and tenure of an employee and the outcome is whether the employee will leave the organization (1) or not (0). We choose to explain class 1.

\noindent \paragraph{HELOC:} This dataset is obtained from \url{https://community.fico.com/s/explainable-machine-learning-challenge?tabset-3158a=2}.
The HELOC dataset had about 10000 instances, and some instances were basically empty and were excluded. From the rest, we used only the first 1000 instances both in the user study (Sections \ref{sec:user_study} and \ref{sec:eval_public_data}) in order to speed up explanation generation. The dataset had two labels - 0 indicating default on loan, and 1 indicating full repayment. We choose to explain class 1. The features correspond to financial health of people who apply for a loan.

\noindent \paragraph{Waveform:} The dataset was obtained from \url{https://archive.ics.uci.edu/ml/datasets/Waveform+Database+Generator+(Version+1)}. The outcome variable is 3 classes of waves (0, 1, 2) and the features are noisy combinations of 2 out of 3 base waves. We choose to explain class 0.

\noindent \paragraph{ATIS:} This dataset contains short texts along with 26 intents (labels) associated to each text, in a air travel information system. This dataset was obtained from \url{https://www.kaggle.com/siddhadev/ms-cntk-atis} and processed using the code provided in \url{https://www.kaggle.com/siddhadev/atis-dataset-from-ms-cntk}. We used the slot filling IOB labels as input features after binarization - any value greater than 0 will be coded as 1. We also removed the last feature (O). The dataset is used for intent classification $26$ intents and we choose to explain the class \textit{flight}. For the purposes of evaluation, we merged the original train and test partitions, and then create train and test partitions using random sampling as described in the main paper.

\section{(Hyper-)Parameter Settings in Experiments}

\noindent \paragraph{Evaluation on Public Datasets and the User Study:} The neighborhood size $|\mathcal{N}_i|$ was set to $10$ when running LIME to generate the leaf explanations. We found this size to be reasonable since larger sizes will make the codes run slower. We also found the default kernel width setting in LIME ($0.75 * \sqrt{p}$) worked well, so we used that. The number of non-zero coefficients in an explanation (a.k.a. explanation complexity) was set at 5, since we found it to be a good number for users to digest multiple explanations. The public datasets were evaluated with only one random split. The training partitions were used to train the black box model and the explanation methods. Results were generated using the test partition.

When running MAME and Two Step, we set the multiplicative step-size $t$ to $1.01$, the initial regularization level for $\beta$, $\epsilon = 1e-10$. We used $10$ conjugate gradient iterations when solving for $\Theta$ in (\ref{eqn:mame_admm_theta}).

\noindent \paragraph{Evaluation for the Expert Study:} The neighborhood size $|\mathcal{N}_i|$ was set to $15$ here. Side information (refer Section 4.1 Study 2 in the main paper) from expert connecting pumps with the same manufacturer type consisted of 2155 edges. The explanation complexity was set at 3  (Figure 2). Figure 1 does not encode any side information and the complexity was set at 4. There were a total of 94 levels in the tree in Figure 1. The remaining parameters were set as mentioned above for the public datasets.

\section{Computing Infrastructure}
We ran the codes in an Ubuntu machine with 64 GB RAM and 32 cores. For the ATIS dataset alone, we used a machine with ~250 GB RAM to avoid memory issues. Both MAME and Two Step were implemented without explicit parallelization, and the parallel operations only happened implicitly in Linear Algebra libraries. The codes were written in Julia 1.3 and also utilized some Python 3.7 libraries (e.g., for RF and MLP model building). 

\section{Note on Two Step method Implementation}
Since the optimization for Two Step method discussed in Section \ref{sec:method} is a special case of that of the MAME method in (\ref{eqn:mame}), we re-purposed the code written for MAME method by re-defining the variables and setting parameters appropriately.

\section{Additional Evaluation on Public Datasets}

The generalized fidelity measures for  the datasets for RF black-boxes is given in Figures \ref{fig:autompgs_RFR_fidelity} - \ref{fig:atis_RFC_fidelity}. Compared to the MLP black-boxes, the first observation we make is that the $R^2$ is much better for RF black-boxes. It seems like MLP models have decision surfaces that are much harder to approximate using linear fits compared to RF models. With RF black-boxes, in the \textit{ATIS} dataset which is fully categorical, MAME still performs better than Two Step and SP-LIME at least for small number of clusters. Similar behavior holds for \textit{Retention} data, which has a lot of categorical features. However, the behavior with \textit{HELOC} and \textit{Waveform} datasets that have only numerical features is mixed. In \textit{Auto MPG}, the results are mixed as well even though it contains categorical features, but again this is a much simpler dataset. A plausible explanation is that when the decision surface of the black-box model is simpler (such as in RF models), just choosing diverse explanations using a method like  SP-LIME is sufficient when the features are fully numerical. However, when the features are categorical or when the decision surfaces are complex (such as in MLP models), a more sophisticated explanation method gives better fidelity. Note that for \textit{ATIS}, SP-LIME took an inordinately long time, so we ran it only for 1 to 1000 representative explanations.

\begin{figure*}[t]
     \centering
     \footnotesize
     \begin{subfigure}[b]{0.49\textwidth}
        \centering
         \includegraphics[width=0.8\textwidth]{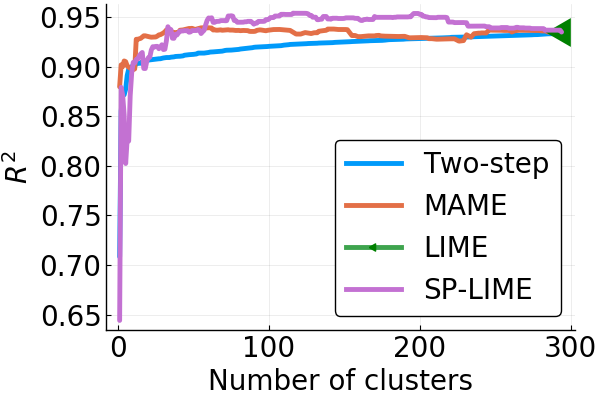}
         \caption{Auto MPG}
         \label{fig:autompgs_RFR_fidelity}
     \end{subfigure}
     \begin{subfigure}[b]{0.49\textwidth}
        \centering
         \includegraphics[width=0.8\textwidth]{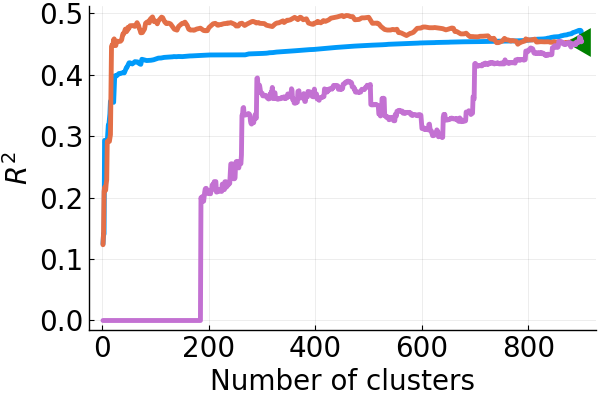}
         \caption{Retention}
         \label{fig:retention_1200_RFC_fidelity}
     \end{subfigure}
     \begin{subfigure}[b]{0.49\textwidth}
         \centering
         \includegraphics[width=0.8\textwidth]{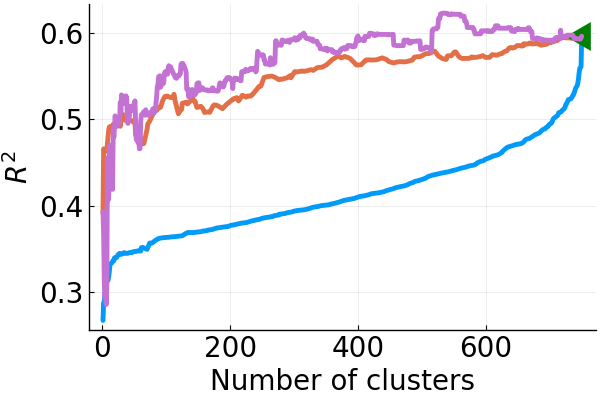}
         \caption{HELOC}
         \label{fig:heloc_RFC_fidelity}
     \end{subfigure}
     \begin{subfigure}[b]{0.49\textwidth}
         \centering
         \includegraphics[width=0.8\textwidth]{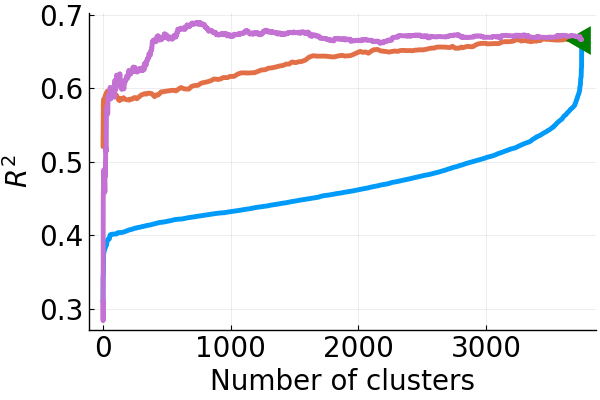}
         \caption{Waveform}
         \label{fig:waveform_RFC_fidelity}
     \end{subfigure}
     \begin{subfigure}[b]{\textwidth}
         \centering
         \includegraphics[width=0.4\textwidth]{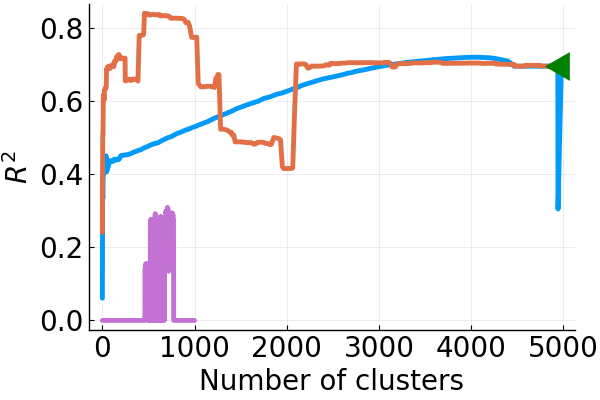}
         \caption{ATIS}
         \label{fig:atis_RFC_fidelity}
     \end{subfigure}     
        \caption{Generalized fidelity measure for the datasets with RF classifier. The legend in (a) is applicable to  all figures.}
        \label{fig:generalized_fidelity_RF}
\end{figure*}

\section{Approximation Quality and Timing Comparisons between the Exact and AR-based methods}
In order to demonstrate the approximation quality between Exact and AR-based solutions, we plot the approximate quality and timing between the two methods. Note that the exact method runs the ADMM iterations in (\ref{eqn:mame_admm_theta})-(\ref{eqn:mame_admm_Z2}) several times for each $\beta$ value until convergence, whereas the AR-based method runs the iterations only once for each $\beta$ value.

We use the \textit{Auto MPG} dataset (both train and test partition together) to obtain a RF regressor black box model and train MAME trees. We set $\epsilon = 1e-10$ and choose  $t$ from $\{1.01, 1.05, 1.1, 1.2, 1.3, 1.4, 1.5\}$. We therefore run the exact and AR-based methods 7 times each, one for each value of $t$. Both the exact and AR-based methods are warm-started using the solution for previous value of $\beta$. The approximation quality is measure by a normalized version of the measure given in Theorem $\ref{thm:mamehausdorffmain}$. The normalization factor that divides the measure is $pn\mu$ where $\mu = \max_{i,j: (i, j) \in \mathcal{E}} \|\theta_i^{(0)} - \theta_j^{(0)} \|_2 $. Note that the superscript $(0)$ indicates that the explanations belong to the leaf nodes.

From Figure \ref{fig:distance_exact_AR}, we see that the exact and approximate solutions get closer as $t \rightarrow 1$ as predicted by the theory. We also see from Figure \ref{fig:time_exact_AR} that the AR-based solution is around 10 times faster to compute than the exact solution for all $t$ values. Both these results demonstrate the utility of the AR-based approximate method.

\begin{figure}[htbp]
\centering
\includegraphics[width=0.5\textwidth]{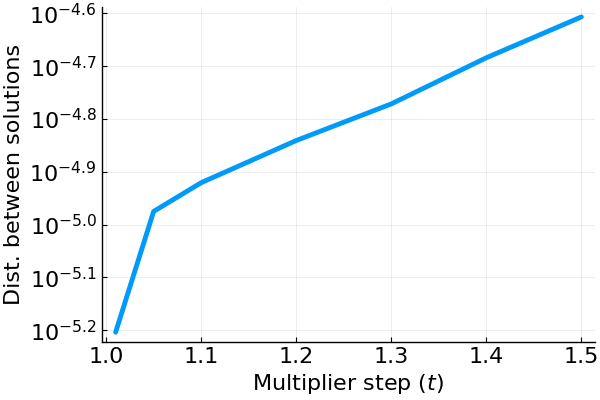}
\caption{Normalized distance between exact and AR-based solutions.}
\label{fig:distance_exact_AR}
\end{figure}

\begin{figure}[htbp]
\centering
\includegraphics[width=0.5\textwidth]{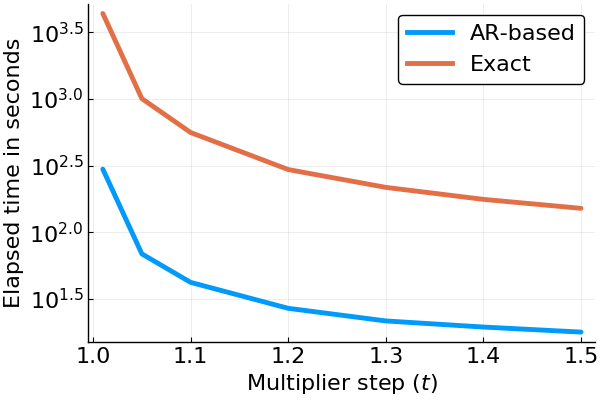}
\caption{Runtime comparison between exact and AR-based solutions.}
\label{fig:time_exact_AR}
\end{figure}

\section{Proof Sketch of Lemma 3.1}
\begin{proof}[Proof Sketch]
If $O_p$ denotes the objective in equation \ref{eqn:mame} being optimized at level $p$, then $O_p = O_{p-1}+\Delta_p\left(\sum_{i < j} w_{ij} \|\theta_i - \theta_j\|_2 \right)$, where $\Delta_p=\beta_p-\beta_{p-1}$. We know that $\Delta_p>0$ by design and so we have an added penalty.

If at the optimal of level $p$, $\|\theta_i^{(p)} - \theta_j^{(p)}\|_2 > \|\theta_i^{({p-1})} - \theta_j^{({p-1})}\|_2$ for some $x_i$ and $x_j$, then that would imply that the other two terms in the objective reduce enough to compensate for the added penalty. However, this would imply that $\theta_i^{({p-1})}$ and $\theta_j^{({p-1})}$ were not the optimal solution at level $p-1$ as the current solution would be better given the lesser emphasis on the last term (i.e. lesser $\beta$) at that level. This contradicts our assumption.
\end{proof}

\section{Linear Convergence of ADMM for MAME}
\label{sec:linearadmmmame}
Strong convexity of the objective function ensures linear convergence of the ADMM method in practice. However, in the absence of strong convexity certain additional criteria need to be satisfied to have linear convergence. Hong and Luo in their paper~\cite{hong2017linear} established the global linear convergence of ADMM for minimizing the sum of any number of convex separable functions expressed in the form below. 
\begin{align}
\label{eqn:hongluo2017}
\text{minimze} \quad f(x) = f_1(x_1) + f_2(x_2) + \ldots + f_K(x_K)\\
\text{subject to} \quad Ex=E_1x_1 + E_2x_2 + \ldots + E_Kx_K=q \nonumber \\
x_k \in X_k, k=1,2,\ldots, K \nonumber
\end{align}

The major assumptions imposed by the authors in their paper to prove the linear convergence for the function $f(x)$ are the following

\begin{enumerate}
    \item 
    Each $f_k$ can further be decomposed as $f_k(x_k)$=$g_k(A_kx_k)$ + $h_k(x_k)$ where $g_k$ and $h_k$ are both convex and continuous over their domains.
    \item
    Each $g_k$ is strictly convex and continuously differentiable with a uniform Lipschitz continuous gradient.
    \begin{align}
    \|\nabla A_k^Tg_k(Ax_k) - A_k^T\nabla g_k(Ax^{'}_k) \| \le L \| x_k - x^{'}_k \| \nonumber
    \end{align}
    \item
    Each $h_k$ satisfies either one of the following conditions
    \begin{itemize}
        \item The epigraph of $h_k(x_k)$ is a polyhedral set.
        \item 
        $h_k(x_k) = \lambda_k \|x_k\|_1 + \sum_{J}{w_{J}\|x_{k,J}\|_2}$ where $x_k = (\ldots,x_{k,J},\ldots)$ is a partition of $x_k$ with $J$ being the partition index.
        \item Each $h_k(x_k)$ is the sum of the functions described in the previous two terms.
    \end{itemize}
    \item
    For any fixed and finite y and $\xi$, $\sum_{k}{h_k(x_k)}$ is finite for all $x \in \lbrace x: L(x;y) \le \xi \rbrace \cap X $.
    \item
    Each submatrix $E_k$ has full column rank.
    \item
    The feasible sets $X_k, k=1,\ldots ,K$ are compact polydedral sets.
\end{enumerate}

In Equation~\ref{eqn:hongluo2017}, each $f_k$ is a convex function subject to linear equality constraints. Our original MAME problem in Equation~\ref{eqn:mame_1} can be written in the form below as in Equation~\ref{eqn:mame_hongluo}, which satisfies all the necessary criteria specified above from 1-6. For example, it is known that the epigraph of the $\ell_1$ norm is a polydedral set. The $\mathrm{I}$ identity matrix and the difference matrix $D$ satisfy the full column rank condition also. 

Hong and Luo \cite{hong2017linear} additionally mentions that each $f_k$ may only consist of the convex non-smooth function $h_k$ and the strongly convex part $g_k$ can be absent. This helps us ensure that the two sparsity inducing norm terms in our formulation on $U$ and $V$ do not violate any of the conditions stated. To understand the proof of convergence, we recommend the readers to go through the proof provided in Hong and Luo \cite{hong2017linear}. This completes the proof for linear convergence of ADMM for MAME.
\begin{align}
\label{eqn:mame_hongluo}
\argmin_{\Theta, U, V} &\sum_{i = 1}^n \sum_{z \in \mathcal{N}_i}\psi\left(x_i,z\right) \left(f(z) - g(z)^T \theta_i\right)^2 \\
&+ \alpha_i ||U_{.i}||_1 + \beta \left(\sum_{e_l \in \mathcal{E}} w_l \|V_{.l}\|_2 \right),\nonumber \\
&\text{ such that } \Theta (\mathrm{I} + D) - \mathrm{I}U -\mathrm{I}V=0,\nonumber
\end{align}

\section{Proof of Theorem 3.1 (from main paper)}
\label{app:proof}

In this section we prove Theorem 3.1 on the expectation of difference between exact and AR solutions. We begin with 3 technical lemmas based on the lemmas given in ~\cite{weylandt2019dynamic} which maybe of independent interest: Lemma \ref{lem:q_linear} provides a convergence rate for the optimization step embedded within an iteration; Lemma \ref{lem:mamelip_paths} establishes a form of Lipschitz continuity for convex clustering regularization paths; Lemma \ref{lem:mameerr_bound} provides a global bound for the approximation error induced at any iteration. 
\begin{lemma}[Q-Linear Error Decrease] \label{lem:q_linear}
At each iteration $k$, the approximation error decreases by a
factor $c < 1$ not depending on $t$ or $\epsilon$. That is,
\[\|\Theta^{(k)} - \Theta^{(\beta)}\| < c \left[\|\Theta^{(k-1)} - \Theta^{(\beta)}\|\right]\]
for some $c$ strictly less than 1 where $\beta$ is the computed regularization parameter on the AR path ($\beta=\gamma^{(k)}$).
\end{lemma}
\begin{proof}
In the notation of ~\cite{hong2017linear}, the constraint matrix for MAME problem from Equation~\ref{eqn:mame_hongluo} is given by $E = \begin{pmatrix} I+D & - I & - I \end{pmatrix}$, for appropriately sized identity matrices, which is clearly row-independent (one of the assumptions mentioned in Section~\ref{sec:linearadmmmame}), yielding linear convergence of the primal and dual variables at a rate $c_{\lambda} < 1$ which may depend on $\lambda$. This follows from the proof sketch for the linear convergence of ADMM for MAME which has been provided in Section~\ref{sec:linearadmmmame} where we show how our formulation satisfies all the assumptions stated in \cite{hong2017linear}. Taking $c = \sup_{\lambda \leq \lambda_{\max}} c_{\lambda}$, we observe that the MAME iterates are uniformly Q-linearly convergent at a rate $c$.
\end{proof}

\begin{lemma}[Lipschitz Continuity of Solution Path] \label{lem:mamelip_paths}

  $\Theta^{(\lambda)}$ is $L$-Lipschitz with
  respect to $\lambda$. That is,
  \[\|\Theta^{(\lambda_1)} - \Theta^{(\lambda_2)}\| \leq L*|\lambda_1 - \lambda_2|\]
  for some $L > 0$.

\end{lemma}

\begin{proof} 
We first show that $\Theta^{(\lambda)}$ is Lipschitz. The vectorized version of MAME problem can be written as
\[\theta^{(\lambda)} = \argmin_{\theta \in \mathbb{R}^{n \times p}} \frac{1}{2}\|\theta - u\|_2^2 + \lambda f_p(\theta) +  \lambda f_q(\tilde{D}\theta)\]
where $u = \vecop(\psi^{1/2}f(z))$, $\theta = \vecop(\psi^{1/2}g(z)^T\Theta)$, $f_q$ and $f_p$ are convex functions, and $\tilde{D} = I \otimes D$ is a fixed matrix.
The KKT conditions give
\[0 \in \theta_{\lambda} - u + \lambda\partial f_p(\theta_{\lambda}) +  \lambda \tilde{D}^T\partial f_q(\tilde{D}\theta_{\lambda}) \]
where $\partial f_p(\cdot)$, $\partial f_q(\cdot)$ are the subdifferential of $f_p$ and $f_q$. Since both $f_p$ and $f_q$ are
convex, it is differentiable almost everywhere [Theorem 25.5]~\cite{rockafellar1970convex},
so the following holds for almost all $\theta_{\lambda}$:
\[0 = \theta_{\lambda} - u + \lambda f_p'(\theta_{\lambda}) + \lambda \tilde{D}^T f_q'(\tilde{D}\theta_{\lambda}) \]
Differentiating with respect to $\lambda$, we obtain
\begin{align*}
  0 &= \theta_{\lambda} - u + \lambda f_p'(\theta_{\lambda}) + \lambda \tilde{D}^T f_q'(\tilde{D}\theta_{\lambda}) \\
  \frac{\partial}{\partial \lambda}\left[0\right] &= \frac{\partial}{\partial \lambda}\left[\theta_{\lambda} - u + \alpha f_p'(\theta_{\lambda}) + \beta \tilde{D}^T f_q'(\tilde{D}\theta_{\lambda})\right] \\
  0 &= \frac{\partial \theta_{\lambda}}{\partial \lambda} - 0 + \lambda \frac{\partial}{\partial \lambda}\left[f'_p(\theta_{\lambda})\right] + f'_p(\theta_{\lambda}) \\
  &+ \lambda \frac{\partial}{\partial \lambda}\left[\tilde{D}^Tf'_q(\tilde{D}\theta_{\lambda})\right] + \tilde{D}^Tf'_q(\tilde{D}\theta_{\lambda}) \\
  0 &= \frac{\partial \theta_{\lambda}}{\partial \lambda} + \lambda f_p''(\theta_{\lambda})\frac{\partial \theta_{\lambda}}{\partial \lambda} + f'_p(\theta_{\lambda}) \\
  &+\lambda \tilde{D}^Tf_q''(\tilde{D}\theta_{\lambda})\tilde{D}\frac{\partial \theta_{\lambda}}{\partial \lambda} + \tilde{D}^Tf'_q(\tilde{D}\theta_{\lambda}) \\
\implies \frac{\partial \theta}{\partial \lambda} &= -[I + \lambda f''_p(\theta)]^{-1} f'_p(\theta) \\
&-[I + \lambda\tilde{D}^T f''_q(\tilde{D}\theta) \tilde{D}]^{-1} D^Tf'_q(\tilde{D}\theta).
\end{align*}
Note that $\theta_{\lambda}$ depends on $\lambda$ so the chain rule must be used here. From here, we note
\[\left\|\frac{\partial \theta_{\lambda}}{\partial \lambda}\right\|_{\infty}\]

$\leq\| -[I + 0]^{-1} f'_p(\theta_{\lambda})-[I + 0]^{-1} \tilde{D}^Tf'_q(\tilde{D}\theta_{\lambda}) \|_{\infty}$ \\
$= \|f'_p(\theta_{\lambda}) + \tilde{D}^T f'_q(\tilde{D}\theta_{\lambda})\|_{\infty}.$

For the MAME problem, we recall that $f_p(\cdot)$ and $f_q(\cdot)$ are convex norms and hence have bounded gradients; hence $f'_p(\theta_{\lambda})$ and $f'_q(\tilde{D}\theta_{\lambda})$ are bounded so the gradient of the regularization path is bounded and exists almost everywhere. This implies that the regularization path is \emph{piecewise} Lipschitz. Since the solution path is constant for $\lambda \geq \lambda_{\max}$ and is continuous, the solution path is globally Lipschitz with a Lipschitz modulus equal to the maximum of the piecewise Lipschitz moduli.
\end{proof}

\begin{lemma}[Global Error Bound] \label{lem:mameerr_bound}
The following error bound holds for all $k$:
\[\|\Theta^{(k)} - \Theta^{(\beta)}\| \leq c^kL\epsilon + L(t - 1)\epsilon t^k \sum_{i=1}^{k-1} \left(\frac{c}{t}\right)^i\]
\end{lemma}

\begin{proof}
Our proof proceeds by induction on $k$. First note that, at initialization:
\[\|\Theta^{(0)} - \Theta^{(\epsilon)}\| \leq L\epsilon\]
by Lemma \ref{lem:mamelip_paths}.

Next, at $k=1$, we note that
\[\|\Theta^{(1)} - \Theta^{(t\epsilon)}\| \leq c \|\Theta^{(0)} - \Theta^{(t\epsilon)}\| \]
by Lemma \ref{lem:q_linear}. We now use the triangle inequality to split the right hand side:
\[\|\Theta^{(0)} - \Theta^{(t\epsilon)}\| \leq \underbrace{\|\Theta^{(0)} - \Theta^{(\epsilon)}\|}_{\text{RHS-1}} + \underbrace{\|\Theta^{(\epsilon)} - \Theta^{(t\epsilon)}\|}_{\text{RHS-2}}\]

From above, we have $\text{RHS-1} \leq L\epsilon$. Using Lemma \ref{lem:mamelip_paths},
$\text{RHS-2}$ can be bounded by
\[\|\Theta^{(\epsilon)} - \Theta^{(t\epsilon)}\| \leq L \left|t\epsilon - \epsilon\right| = L(t-1)\epsilon.\]

Putting these together, we get
\[\|\Theta^{(1)} - \Theta^{(t\epsilon)}\| \leq c \left[\text{RHS-1} + \text{RHS-2}\right] \leq c\left[L\epsilon + L(t-1)\epsilon\right] = cLt\epsilon \]
Repeating this argument for $k=2$, we see
\begin{align*}
\|\Theta^{(2)} - \Theta^{(t^2\epsilon)}\| &\leq c \|\Theta^{(1)} - \Theta^{(t^2\epsilon)}\| \\
                                        &\leq c\left[\|\Theta^{(1)} - \Theta^{(t\epsilon)}\| + \|\Theta^{(t\epsilon)} - \Theta^{(t^2\epsilon)}\|\right] \\
                                        &\leq c\left[cLt\epsilon + L\left|t^2\epsilon - t\epsilon\right|\right] \\
                                        &= c^2Lt\epsilon + cL(t-1)\epsilon * t \\
                                        &= c^2Lt\epsilon + L\epsilon(t-1)t^2 * \left(\frac{c}{t}\right) \\
                                        &= c^2Lt\epsilon + L\epsilon(t-1)t^2 * \sum_{i=1}^{k-1} \left(\frac{c}{t}\right)^i
                                        \end{align*}

We use this as a base case for our inductive proof and prove the  general case:
\begin{align*}
\|\Theta^{(k)} - \Theta^{(t^k\epsilon)}\| &\leq c \|\Theta^{(k-1)} - \Theta^{(t^k\epsilon)}\| \\
                                        &\leq c\left[\|\Theta^{(k-1)} - \Theta^{(t^{k-1}\epsilon)}\| + \|\Theta^{(t^{k-1}\epsilon)} - \Theta^{(t^k\epsilon)}\|\right] \\
                            &\leq c\left[c^{k-1}Lt\epsilon + L\epsilon(t-1)t^{k-1} \sum_{i=1}^{k-2} \left(\frac{c}{t}\right)^i + L\left|t^k\epsilon - t^{k-1}\epsilon\right|\right] \\
                                        &= c^{k}Lt\epsilon + cL\epsilon(t-1)t^{k-1}\sum_{i=1}^{k-2}\left(\frac{c}{t}\right)^i + cL\epsilon(t^k - t^{k-1}) \\
                                        &= c^{k}Lt\epsilon + L\epsilon(t-1)t^k\left[\frac{c}{t}\sum_{i=1}^{k-2}\left(\frac{c}{t}\right)^i + \frac{c}{t}\right] \\
                                        &= c^{k}Lt\epsilon + L\epsilon(t-1)t^k\left[\sum_{i=2}^{k-1}\left(\frac{c}{t}\right)^i + \frac{c}{t}\right] \\
                                        &= c^{k}Lt\epsilon + L\epsilon(t-1)t^k\sum_{i=1}^{k-1}\left(\frac{c}{t}\right)^i
                                        \end{align*}

With these results, we are now ready to prove Theorem 3.1
\end{proof}

\begin{proof}
We begin by fixing temporarily $\beta$ and bounding
\[\inf_k \left\|\Theta^{(k)} - \Theta^{(\beta)}\right\|\]
The infimum over all $k$ is less than the distance at any particular $k$,
so it suffices to choose a value of $k$ which gives convergence to 0. Let $\tilde{k}$ be the value of $k$ which
gives the closest value of $\gamma^{(k)}$ to $\beta$ along the AR path; and let
$\tilde{\beta} = \gamma^{(\tilde{k})} = \epsilon t^{\tilde{k}}$. That is,
\[\tilde{k} = \argmin_k |\gamma^{(k)} - \beta| \quad \text{ and } \tilde{\beta} = \gamma^{(\tilde{k})}\]
Then
\[\inf_k \left\|\Theta^{(k)} - \Theta^{(\beta)}\right\| \leq \|\Theta^{(\tilde{k})} - \Theta^{(\beta)}\| \leq \underbrace{\|\Theta^{(\tilde{k})} - \Theta^{\tilde{(\beta)}}\|}_{\text{RHS-1}} + \underbrace{\|\Theta^{\tilde{(\beta)}} - \Theta^{(\beta)}\|}_{\text{RHS-2}}\]

Using Lemma \ref{lem:mamelip_paths}, we can bound $\text{RHS-2}$ as
\begin{align}
\nonumber
\text{RHS-2} &\leq L |\tilde{\beta} - \beta| \leq L | \gamma^{(\tilde{k} + 1)} - \gamma^{(\tilde{k} - 1)}| = L * \epsilon t^{\tilde{k}-1} * [t^2-1] \\
\nonumber
&\leq L * \beta_{\max} * [t^2 -1]
\end{align}

Using Lemma \ref{lem:mameerr_bound}, we can bound $\text{RHS-1}$ as
\begin{align}
\text{RHS-1} & \le c^{\tilde{k}}L\epsilon + L(t-1) * \epsilon t^{\tilde{k}} \sum_{i=1}^{k-1} \left(\frac{c}{t}\right)^i
\nonumber \\
&\leq c^{\tilde{k}}L\epsilon + L(t-1) * \epsilon t^{\tilde{k}} * C \nonumber
\label{eqn:rhs_bound}
\end{align} where $C = \sum_{i=1}^{\infty} \left(\frac{c}{1+t}\right)^i$ is large but finite.
Since $c < 1$ and $\tilde{\beta} = \epsilon t^{\tilde{k}} \leq \beta_{\max}$, we can replace the $k$-dependent
quantities to get 
\[\text{RHS-1} = \|\Theta^{(\tilde{k})} - \Theta^{(\tilde{\beta})}\| \leq L\epsilon + C * L(t-1) * \beta_{\max}\]

Putting these together, we have
\[\inf_{k} \|\Theta^{(k)} - \Theta^{(\beta)}\| \leq \text{RHS-1} + \text{RHS-2}\]
\[\leq L\epsilon + C * L(t-1) * \beta_{\max} + L * \beta_{\max} * [t^2 -1]\]

Similarly by fixing $k$ we know that 
\[\inf_{\beta} \|\Theta^{(k)} - \Theta^{(\beta)}\| \leq \|\Theta^{(k)} - \Theta^{(\tilde{\beta})}\| \]
where 
\[\tilde{\beta} = \argmin_\beta |\log_\gamma \beta - k| \quad \text{ and } \tilde{k} = \log_\gamma \tilde{{\beta}}\]

Using similar arguments as above based on Lemma \ref{lem:mameerr_bound} we can show that 
\[\inf_{\beta} \|\Theta^{(k)} - \Theta^{(\beta)}\| \leq \|\Theta^{(k)} - \Theta^{(\tilde{\beta})}\| \leq L\epsilon + C * L(t-1) * \beta_{\max}\]

We can use these two bounds and plug them in here
\[
\max\Big\{ 
    E_\beta \left(\inf_k \|\Theta^{(k)} - \Theta^{(\beta)} \|\right),
    E_k \left(\inf_\beta \|\Theta^{(k)} - \Theta^{(\beta)} \| \right)
    \Big\}\]
\[\le\max\Big\{ 
    E_\beta \left(L\epsilon + C * L(t-1) * \beta_{\max} + L * \beta_{\max} * [t^2 -1] \right),\]
    \[E_k \left(L\epsilon + C * L(t-1) * \beta_{\max} \| \right)
    \Big\} \xrightarrow{(t, \epsilon) \to (1, 0)} 0\]

One can observe that as $t$, $\epsilon$ $\rightarrow$ to (1, 0) both the expectation terms reduce to 0 individually. Hence proved
\end{proof}

\section{User Study Material}
Figure~\ref{fig:screenshot} is a screenshot of the user study in the LIME condition. The visual explanations show the feature contribution, likelihood of repayment, and feature importance calculated using the SP-LIME method. Below the explanation is a table showing the 22 financial records for a loan applicant. It is divided into two segments. The top segment shows the features used in the visual explanation, which are also the most important features deemed by the explanation algorithm. The bottom segment shows the rest of the features. At the bottom of the page are two questions that the participant had to answer for every trial. The participant could access a written instruction any time during the experiment by clicking the "SEE INSTRUCTIONS" button at the top of the web page.
\begin{figure*}
\centering
\includegraphics[width=0.8\textwidth]{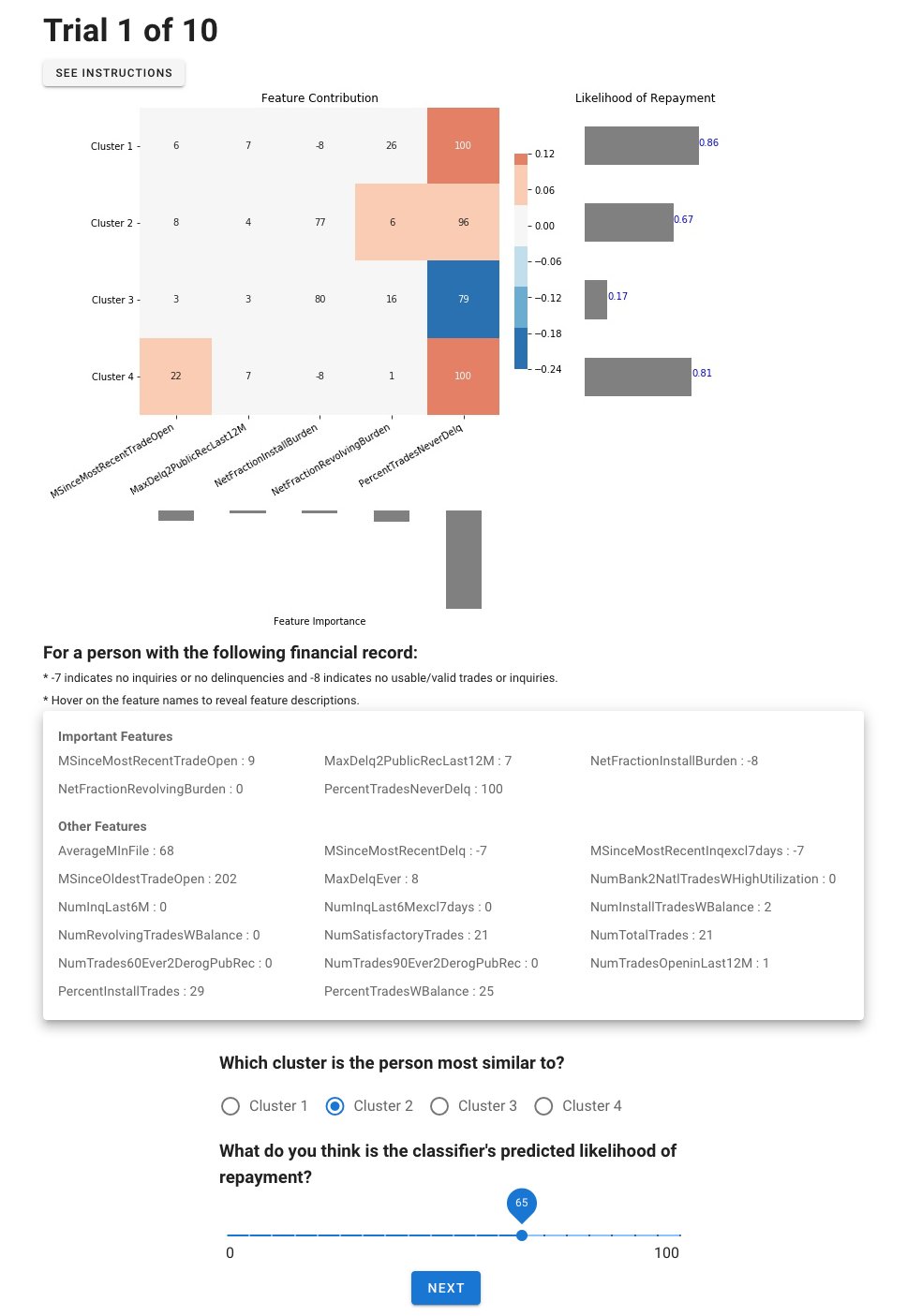}
\caption{A screenshot of the user study in the LIME condition.}
\label{fig:screenshot}
\end{figure*}

\section{Computational Complexity Analysis}
We provide the computational complexity for running one set of iterations given in (\ref{eqn:mame_admm_theta})-(\ref{eqn:mame_admm_Z2}). This is the same as obtaining solutions for one $\beta$ value if we use the AR-based method. Let us consider the five steps individually.

For (\ref{eqn:mame_admm_theta}), we use conjugate gradients to obtain the solution. If we assume the number of edges $|\mathcal{E}| = O(n)$, and the number of CG iterations to be $s$, the dominant complexity of this step is $O(p^2 n s) + O(p n^2 s)$. For (\ref{eqn:mame_admm_U}), the update involves a soft-thresholding step which incurs a complexity of $O(pn)$. The update step (\ref{eqn:mame_admm_V}), similarly incurs a complexity of $O(pn)$, assuming $|\mathcal{E}| = O(n)$. Updates for $Z_1$ and $Z_2$ in (\ref{eqn:mame_admm_Z1}) and (\ref{eqn:mame_admm_Z2}) respectively involve complexities of $O(pn)$ each. If we assume $s$ to be very small (we use $10$ in our experiments), the dominant complexity for one ADMM iteration is hence $O(pn(p+n))$.
